\let\oldvec\vec
\documentclass[twocolumn]{svjour3}
\let\vec\oldvec

\smartqed

\usepackage{newtxtext,newtxmath}
\usepackage{amssymb,amsmath}
\usepackage{pstricks,pst-node}
\usepackage[pagebackref=true,breaklinks=true,colorlinks,bookmarks=false]{hyperref}

\newcommand{\pb}[1]{\psshadowbox[linewidth=0.4pt,linecolor=black,shadowcolor=lightgray]{#1}}
\newcommand{\ps}[1]{\psshadowbox[linewidth=0.4pt,fillcolor=red!10,fillstyle=solid,linecolor=black,shadowcolor=lightgray]{#1}}
\newcommand{\pc}[1]{\psshadowbox[linewidth=0.4pt,fillcolor=green!10,fillstyle=solid,linecolor=black,shadowcolor=lightgray]{#1}}
\newcommand{\degen}{\ps{\parbox{40pt}{\centering $\widetilde E_{ij} = 0_{3\times 3}$}}}
\newcommand{\collin}[1]{\pc{\parbox{45pt}{\centering collinear\\ compatible\\triplet~\eqref{#1}}}}
\newcommand{\compat}[1]{\pc{\parbox{45pt}{\centering compatible\\ triplet~\eqref{#1}}}}
\newcommand{\compator}[2]{\pc{\parbox{45pt}{\centering compatible\\ triplet~\eqref{#1}\\ or~\eqref{#2}}}}
\newcommand{\bx}[1]{\pb{#1}}
\psset{linearc=0.2}

\DeclareMathOperator{\tr}{tr}
\DeclareMathOperator{\diag}{diag}

\journalname{}

\begin{document}

\title{Necessary and Sufficient Polynomial Constraints on Compatible Triplets of Essential Matrices\thanks{The work was supported by Act 211 Government of the Russian Federation, contract No.~02.A03.21.0011.}}

\author{E.V. Martyushev}

\institute{E.V. Martyushev \at
South Ural State University, 76 Lenin Avenue, Chelyabinsk 454080, Russia \\
\email{martiushevev@susu.ru}
}

\date{Received: date / Accepted: date}

\maketitle

\begin{abstract}

The essential matrix incorporates relative rotation and translation parameters of two calibrated cameras. The well-known algebraic characterization of essential matrices, i.e. necessary and sufficient conditions under which an arbitrary matrix (of rank two) becomes essential, consists of a unique matrix equation of degree three. Based on this equation, a number of efficient algorithmic solutions to different relative pose estimation problems have been proposed. In three views, a possible way to describe the geometry of three calibrated cameras comes from considering compatible triplets of essential matrices. The compatibility is meant the correspondence of a triplet to a certain configuration of calibrated cameras. The main goal of this paper is to give an algebraic characterization of compatible triplets of essential matrices. Specifically, we propose necessary and sufficient polynomial constraints on a triplet of real rank-two essential matrices that ensure its compatibility. The constraints are given in the form of six cubic matrix equations, one quartic and one sextic scalar equations. An important advantage of the proposed constraints is their sufficiency even in the case of cameras with collinear centers. The applications of the constraints may include relative camera pose estimation in three and more views, averaging of essential matrices for incremental structure from motion, multiview camera auto-calibration, etc.


\end{abstract}

\section{Introduction}

In multiview geometry, the essential matrix describes the relative orientation of two calibrated cameras. It was first introduced by Longuet-Higgins in~\cite{longuet81}, where the \emph{essential matrix} was used for the scene reconstruction from point correspondences in two views. Later, the algebraic properties of essential matrices have been investigated in detail in~\cite{FM,Horn90,HF89}. The well-known characterization of the algebraic variety of essential matrices, which is the closure of the set of essential matrices in the corresponding projective space, consists of a unique matrix equation of degree three~\cite{Dem88,Maybank}. Based on this equation, a number of efficient algorithmic solutions to different relative pose estimation problems have been proposed in the last two decades, e.g.~\cite{Nister,SEN06,SNKS,kukelova2007two}.

The relative orientation of three calibrated cameras can be naturally described by the so-called \emph{calibrated trifocal tensor}. It first appeared in~\cite{SA,Weng} as a tool of scene reconstruction from line correspondences. The algebraic properties of calibrated trifocal tensors were recently investigated in~\cite{Mart17}, where some necessary and sufficient polynomial constraints on a real calibrated trifocal tensor have been derived. A $3$-view analog of the variety of essential matrices -- the calibrated trifocal variety -- has been introduced in~\cite{Kileel17}, where it was used to compute algebraic degrees of numerous $3$-view relative pose problems. The complete characterization of the calibrated trifocal variety is an open challenging problem. For the uncalibrated case such characterization consists of $10$ cubic, $81$ quintic, and $1980$ sextic polynomial equations~\cite{AO}.

Another way to describe the geometry of three calibrated cameras comes from considering compatible \emph{triplets of essential matrices}. The compatibility means that a triplet must correspond to a certain configuration of three calibrated cameras. In the recent paper~\cite{Kasten19}, the authors considered compatible $n$-view multiplets of essential matrices and proposed their algebraic characterization in terms of either the spectral decomposition or the singular value decomposition of a symmetric $3n\times 3n$ matrix constructed from the multiplet.

On the other hand, \emph{polynomial constraints} on compatible triplets of essential matrices have not previously been studied. The present paper is a step in this direction. Its main contribution is a set of necessary and sufficient polynomial constraints on a triplet of real rank-two essential matrices that ensure the compatibility of the triplet. An important advantage of the proposed constraints is their sufficiency even in the case of cameras with collinear centers. The results of the paper may be applied to different computer vision problems including relative camera pose estimation from three and more views, averaging of essential matrices for incremental structure from motion, multiview camera auto-calibration, etc.

The rest of the paper is organized as follows. In Sect.~\ref{sec:prel}, we recall some definitions and results from multiview geometry. In Sect.~\ref{sec:tripletE}, we propose our necessary and sufficient constraints on compatible triplets of essential matrices. The constraints have the form of $82$ cubic, one quartic, and one sextic homogeneous polynomial equations in the entries of essential matrices. Sect.~\ref{sec:application} outlines two possible applications of the proposed constraints. In Sect.~\ref{sec:disc}, we discuss the results of the paper and make conclusions. The Appendix includes some auxiliary lemmas that we used throughout the proof of our main Theorem~\ref{thm:sufE}.

\section{Preliminaries}
\label{sec:prel}

\subsection{Notation}

We preferably use $\alpha$, $\beta, \ldots$ for scalars, $a$, $b, \ldots$ for column $3$-vectors or polynomials, and $A$, $B, \ldots$ both for matrices and column $4$-vectors. For a matrix $A$ the entries are $(A)_{ij}$, the transpose is $A^\top$, and the adjugate (i.e. the transposed matrix of cofactors) is $A^*$. The determinant of $A$ is $\det A$ and the trace is $\tr A$. For two $3$-vectors $a$ and $b$ the cross product is $a\times b$. For a vector $a$ the entries are $(a)_i$, the notation $[a]_\times$ stands for the skew-symmetric matrix such that $[a]_\times b = a \times b$ for any vector~$b$. We use $I$ for the identity matrix and $\|\cdot\|$ for the Frobenius norm.

\subsection{Fundamental matrix}

We briefly recall some definitions and results from multiview geometry, see~\cite{Faugeras93,FM,HZ,Maybank,arie12} for details.

Let there be given two finite projective cameras $P_1 = \begin{bmatrix} A_1 & a_1 \end{bmatrix}$ and $P_2 = \begin{bmatrix} A_2 & a_2 \end{bmatrix}$, where $A_1$, $A_2$ are invertible $3\times 3$ matrices and $a_1$, $a_2$ are $3$-vectors. Let $Q$ be a point in $3$-space represented by its homogeneous coordinates and $q_j$ be its $j$th image. Then,
\begin{equation}
q_j \sim P_j Q,
\end{equation}
where $\sim$ means an equality up to a scale. The \emph{epipolar constraint} for a pair $(q_1, q_2)$ says
\begin{equation}
\label{eq:epiF}
q_2^\top F_{21} q_1 = 0,
\end{equation}
where~\cite{sengupta17,Kasten19}
\begin{equation}
F_{21} = (A_2^*)^\top[A_2^{-1}a_2 - A_1^{-1} a_1]_\times A_1^*
\end{equation}
is called the \emph{fundamental matrix}. By definition, $F_{21}$ must be of rank two. The left and right null vectors of $F_{21}$, that is vectors
\begin{equation}
\label{eq:epipoles}
e_{21} = A_2A_1^{-1}a_1 - a_2 \quad\text{and}\quad e_{12} = A_1A_2^{-1}a_2 - a_1
\end{equation}
respectively, are called the \emph{epipoles}. We notice that $e_{ij}$ is the projection of the $j$th camera center onto the $i$th image plane.

The following theorem gives an algebraic characterization of the set of fundamental matrices.
\begin{theorem}[\cite{HZ}]
\label{thm:fund}
Any real $3\times 3$ matrix of rank two is a fundamental matrix.
\end{theorem}

\subsection{Essential matrix}

The \emph{essential matrix} $E_{21}$ is the fundamental matrix for calibrated cameras $\hat P_1 = \begin{bmatrix} R_1 & t_1 \end{bmatrix}$ and $\hat P_2 = \begin{bmatrix} R_2 & t_2\end{bmatrix}$, where $R_1, R_2 \in \mathrm{SO}(3)$ and $t_1$, $t_2$ are $3$-vectors, that is
\begin{equation}
\label{eq:essential}
E_{21} = R_2[R_2^\top t_2 - R_1^\top t_1]_\times R_1^\top.
\end{equation}
The proof of formula~\eqref{eq:essential} can be found in~\cite{arie12}.

Camera matrices $P_j$ and $\hat P_j$ are related by
\begin{equation}
P_j \sim K_j \hat P_j H,
\end{equation}
where $K_j$ is an invertible upper-triangular matrix known as the \emph{calibration matrix} of $j$th camera and $H$ is a certain invertible $4\times 4$ matrix. Then it can be shown that
\begin{equation}
\label{eq:F}
F_{21} \sim K_2^{-\top} E_{21} K_1^{-1}.
\end{equation}
Hence epipolar constraint~\eqref{eq:epiF} for the essential matrix becomes
\begin{equation}
\label{eq:epiE}
\hat q_2^\top E_{21} \hat q_1 = 0,
\end{equation}
where $\hat q_j = K_j^{-1} q_j$.

The following theorem gives an algebraic characterization of the set of essential matrices.
\begin{theorem}[\cite{Dem88,FM,Maybank}]
\label{thm:esse}
A real $3\times 3$ matrix $E$ of rank two is an essential matrix if and only if $E$ satisfies
\begin{equation}
\label{eq:esse}
E^\top EE^\top - \frac{1}{2}\tr(E^\top E)\, E^\top = 0_{3\times 3}.
\end{equation}
\end{theorem}

\subsection{Compatible Triplets of Fundamental Matrices}
\label{ssec:tripletF}

A triplet of fundamental matrices $\{F_{12}, F_{23}, F_{31}\}$ is said to be \emph{compatible} if there exist matrices $B_1, B_2, B_3 \in \mathrm{GL}(3)$ and $3$-vectors $b_1$, $b_2$, $b_3$ such that
\begin{equation}
\label{eq:Fcompat}
F_{ij} = B_i^\top [b_i - b_j]_\times B_j
\end{equation}
for all distinct $i, j \in \{1, 2, 3\}$. It follows from~\eqref{eq:Fcompat} that $F_{ji} = F_{ij}^\top$.

\begin{theorem}
\label{thm:sufF}
Let $\mathcal F = \{F_{12}, F_{23}, F_{31}\}$ be a triplet of real rank-two fundamental matrices with non-collinear epipoles, i.e. $e_{ki} \not\sim e_{kj}$ for all $i, j, k \in \{1, 2, 3\}$. Then, $\mathcal F$ is compatible if and only if it satisfies
\begin{equation}
\label{eq:sufF}
F_{ij}^* F_{ik} F_{jk}^* = 0_{3\times 3}.
\end{equation}
\end{theorem}

\begin{proof}
We note that constraints~\eqref{eq:sufF} imply (and are implied by)
\begin{equation}
\label{eq:eFe}
e_{ij}^\top F_{ik} e_{kj} = 0
\end{equation}
for all distinct indices $i, j, k \in \{1, 2, 3\}$. Geometrically, constraints~\eqref{eq:eFe} mean that the epipoles $e_{ij}$ and $e_{kj}$ are matched and correspond to the projections of the $j$th camera center onto the $i$th and $k$th image plane respectively. The proof of compatibility of three fundamental matrices with non-collinear epipoles satisfying~\eqref{eq:eFe} can be found in~\cite{HZ}.
\end{proof}

The main goal of this paper is to propose a generalized analog of Theorem~\ref{thm:sufF} for a triplet of essential matrices with possibly collinear epipoles.

\section{Compatible Triplets of Essential Matrices}
\label{sec:tripletE}

A triplet of essential matrices $\{E_{12}, E_{23}, E_{31}\}$ is said to be \emph{compatible} if there exist matrices $R_1, R_2, R_3 \in \mathrm{SO}(3)$ and $3$-vectors $b_1$, $b_2$, $b_3$ such that
\begin{equation}
\label{eq:Ecompat}
E_{ij} = R_i[b_i - b_j]_\times R_j^\top
\end{equation}
for all distinct $i, j \in \{1, 2, 3\}$.

Given a triplet of essential matrices $\{E_{12}, E_{23}, E_{31}\}$, we denote
\begin{equation}
\label{eq:9x9matrixE}
E = \begin{bmatrix} 0_{3\times 3} & E_{12} & E_{13} \\ E_{21} & 0_{3\times 3} & E_{23} \\ E_{31} & E_{32} & 0_{3\times 3} \end{bmatrix}.
\end{equation}
The symmetric $9\times 9$ matrix $E$ (as well as its analog for fundamental matrices) has been previously introduced in~\cite{sengupta17,Kasten18,Kasten19} where some of its spectral properties were investigated. Matrix $E$ is called compatible if it is constructed from a compatible triplet. In~\cite{Kasten19}, the authors propose necessary and sufficient conditions on the compatibility of matrix $E$ (more precisely, of an $n$-view generalization of matrix $E$) in terms of its spectral or singular value decomposition. In particular, these conditions imply that the characteristic polynomial of a compatible matrix $E$ has the form
\begin{equation}
\label{eq:cpol}
p_E(\lambda) = \lambda^3(\lambda^2 - \lambda_1^2)(\lambda^2 - \lambda_2^2)(\lambda^2 - \lambda_3^2),
\end{equation}
where $\lambda_1$, $\lambda_2$, and $\lambda_3$ are possibly non-zero eigenvalues of~$E$.

Condition~\eqref{eq:cpol} induces polynomial constraints on the entries of matrix $E$. Namely, the coefficients of $p_E(\lambda)$ in $\lambda^6$, $\lambda^4$, $\lambda^2$, $\lambda^1$, and $\lambda^0$ must vanish. It is clear that these coefficients are polynomials in the entries of matrices $E_{12}$, $E_{23}$, and $E_{31}$ of degree $3$, $5$, $7$, $8$, and $9$ respectively. For example, the coefficient in $\lambda^6$ equals $-2\tr(E_{12}E_{23}E_{31})$. Below we propose a set of cubic polynomial equations on matrix $E$ such that constraint~\eqref{eq:cpol} is implied by these equations, see Eqs.~\eqref{eq:necE1}~--~\eqref{eq:necE3}.

We note that condition~\eqref{eq:cpol} alone is not sufficient for the compatibility of a triplet of essential matrices. The eigenvalues of $E$ are additionally constrained by
\begin{equation}
\lambda_1^2 - \lambda_2^2 - \lambda_3^2 = 0,
\end{equation}
where $|\lambda_1| > |\lambda_2|$ and $|\lambda_1| > |\lambda_3|$. It follows that
\begin{equation}
\label{eq:lambda}
(\lambda_1^2 - \lambda_2^2 - \lambda_3^2)(\lambda_2^2 - \lambda_3^2 - \lambda_1^2)(\lambda_3^2 - \lambda_1^2 - \lambda_2^2) = 0.
\end{equation}
Since the l.h.s. of Eq.~\eqref{eq:lambda} is a symmetric function in values $\lambda_1^2$, $\lambda_2^2$, and $\lambda_3^2$, it can be expanded in terms of the elementary symmetric functions which are the coefficients in~\eqref{eq:cpol}. On the other hand, these coefficients can be represented as polynomials in $\tr(E^{2k})$. Thus we get one more polynomial constraint on matrix $E$ of degree six, see Eq.~\eqref{eq:necE5}.

Finally, some of our polynomial constraints are more convenient to formulate using a specific binary operation on the space of $3\times 3$ matrices. Let $A^*$ be the adjugate of a matrix $A$, i.e. its transposed matrix of cofactors. Then for two $3\times 3$ matrices $A$ and $B$ we define
\begin{equation}
A\diamond B = (A - B)^* - A^* - B^*.
\end{equation}
An alternative expression for $A\diamond B$ can be derived using the well-known formula
\begin{equation}
A^* = \frac{1}{2}\,(\tr^2 A - \tr A^2)I - A\tr A + A^2.
\end{equation}
Thus we get
\begin{equation}
A\diamond B = (\tr(AB) - \tr A\tr B)I + A\tr B + B\tr A - AB - BA.
\end{equation}
It is straightforward to show that for any $3\times 3$ matrices $A$, $B$, $C$, $D$ and any scalars $\beta$, $\gamma$ the following identities hold:
\begin{align}
\label{eq:diam1}
A\diamond B &= B\diamond A,\\
A\diamond (\beta B + \gamma C) &= \beta (A\diamond B) + \gamma (A\diamond C),\\
(A\diamond B)^\top &= A^\top\diamond B^\top,\\
(CAD)\diamond (CBD) &= D^*(A\diamond B)C^*,\\
\label{eq:diam5}
A\diamond I &= A - \tr(A)I.
\end{align}

Now we can formulate our polynomial constraints.
\begin{theorem}
\label{thm:necE}
Let $\{E_{12}, E_{23}, E_{31}\}$ be a compatible triplet of essential matrices, matrix $E$ be defined in~\eqref{eq:9x9matrixE}. Then the following equations hold:
\begin{align}
\label{eq:necE1}
\tr(E_{12}E_{23}E_{31}) &= 0,\\
\label{eq:necE2}
E_{ij}^\top E_{ij}E_{jk} - \frac{1}{2}\tr(E_{ij}^\top E_{ij})\, E_{jk} + E_{ij}^*E_{ki}^\top &= 0_{3\times 3},\\
\label{eq:necE3}
E_{jk}^\top E_{ij}^* + E_{jk}^*E_{ij}^\top + (E_{ij}E_{jk}) \diamond E_{ki}^\top &= 0_{3\times 3},\\
\label{eq:necE4}
\tr^2(E^2) - 16\tr(E^4) + 24\sum\limits_{i < j}\tr^2(E_{ij}^\top E_{ij}) &= 0,\\
\label{eq:necE5}
\tr^3(E^2) - 12\tr(E^2)\tr(E^4) + 32\tr(E^6) &= 0
\end{align}
for all distinct $i, j, k \in \{1, 2, 3\}$. There are in total $1 + 6\cdot 9 + 3\cdot 9 = 82$ linearly independent cubic equations of type~\eqref{eq:necE1}~--~\eqref{eq:necE3}.
\end{theorem}

\begin{proof}
Let
\begin{equation}
\label{eq:transform}
\widetilde E_{ij} = U_i E_{ij}U_j^\top,
\end{equation}
where $U_i \in \mathrm{SO}(3)$. It is clear that $\mathcal E = \{E_{12}, E_{23}, E_{31}\}$ is a compatible triplet if and only if so is $\widetilde{\mathcal E} = \{\widetilde E_{12}, \widetilde E_{23}, \widetilde E_{31}\}$. Also it can be readily seen that $\mathcal E$ satisfies Eqs.~\eqref{eq:necE1}~--~\eqref{eq:necE5} if and only if so does~$\widetilde{\mathcal E}$.

Given a compatible triplet $\{E_{12}, E_{23}, E_{31}\}$, where each $E_{ij}$ is represented by~\eqref{eq:Ecompat}, we set $U_i = R_i^\top$. Then essential matrices $\widetilde E_{ij} = U_i E_{ij}U_j^\top$ become skew-symmetric and can be represented in form
\begin{equation}
\widetilde E_{ij} = [b_i - b_j]_\times, \quad \widetilde E_{jk} = [b_j - b_k]_\times, \quad \widetilde E_{ki} = [b_k - b_i]_\times,
\end{equation}
where indices $i, j, k$ are intended to be distinct. We denote $b_i - b_j = c$, $b_j - b_k = d$. This yields $b_k - b_i = -c - d$ and hence we can write
\begin{equation}
\widetilde E_{ij} = [c]_\times, \quad \widetilde E_{jk} = [d]_\times, \quad \widetilde E_{ki} = -[c]_\times - [d]_\times.
\end{equation}
By substituting this into Eqs.~\eqref{eq:necE1}~--~\eqref{eq:necE3}, we get
\begin{multline}
\tr(\widetilde E_{ij}\widetilde E_{jk}\widetilde E_{ki}) = -\tr([c]_\times[d]_\times[c]_\times + [c]_\times[d]_\times[d]_\times) \\= -\tr([c]_\times(cd^\top - (d^\top c)I) + (cd^\top - (d^\top c)I)[d]_\times) \\= (d^\top c)\tr([c + d]_\times) = 0,
\end{multline}
\begin{multline}
\bigl(\widetilde E_{ij}^\top \widetilde E_{ij} - \frac{1}{2}\, \tr(\widetilde E_{ij}^\top \widetilde E_{ij}) I\bigr)\widetilde E_{jk} + \widetilde E_{ij}^* \widetilde E_{ki}^\top \\= \bigl(-[c]_\times[c]_\times + \frac{1}{2}\, \tr([c]_\times[c]_\times) I\bigr)[d]_\times + cc^\top ([c]_\times + [d]_\times) \\= \bigl(-cc^\top + (c^\top c)I - \frac{1}{2}\, 2(c^\top c)I\bigr)[d]_\times + cc^\top [d]_\times \\= -cc^\top[d]_\times + cc^\top [d]_\times = 0_{3\times 3},
\end{multline}
\begin{multline}
\widetilde E_{jk}^\top \widetilde E_{ij}^* + \widetilde E_{jk}^* \widetilde E_{ij}^\top + (\widetilde E_{ij}\widetilde E_{jk}) \diamond \widetilde E_{ki}^\top \\= -[d]_\times cc^\top - dd^\top [c]_\times + (dc^\top - (c^\top d)I) \diamond ([c]_\times + [d]_\times) \\= -[d]_\times cc^\top - dd^\top [c]_\times + (dc^\top)\diamond [c]_\times + (dc^\top)\diamond [d]_\times \\- (c^\top d)([c]_\times + [d]_\times) = -[d]_\times cc^\top - dd^\top [c]_\times \\- [c]_\times dc^\top - dc^\top[d]_\times = 0_{3\times 3}.
\end{multline}
We used that $[a]_\times[b]_\times = ba^\top - (a^\top b)I$ and $([a]_\times)^* = aa^\top$ for arbitrary $3$-vectors $a$ and~$b$. Eqs.~\eqref{eq:necE4}~--~\eqref{eq:necE5} are proved in a similar manner, but the computation is more complicated. Theorem~\ref{thm:necE} is proved.
\end{proof}

\begin{remark}
Constraint~\eqref{eq:sufF} for a triplet of essential matrices is implied by Eq.~\eqref{eq:necE2}. Namely, multiplying~\eqref{eq:necE2} on the right by $E_{jk}^*$ we get~\eqref{eq:sufF}.

We also note that Eq.~\eqref{eq:necE2} is a generalization of Eq.~\eqref{eq:esse}. Setting $k = i$ in~\eqref{eq:necE2} and taking into account that $E_{ii} = 0_{3\times 3}$ we get~\eqref{eq:esse}.
\end{remark}

\begin{theorem}
\label{thm:sufE}
Let $\mathcal E = \{E_{12}, E_{23}, E_{31}\}$ be a triplet of real rank-two essential matrices. Then $\mathcal E$ is compatible if and only if it satisfies Eqs.~\eqref{eq:necE1}~--~\eqref{eq:necE5} from Theorem~\ref{thm:necE}.
\end{theorem}

\begin{proof}
The ``only if'' part is by Theorem~\ref{thm:necE}. Let us prove the ``if'' part.

First of all, we simplify triplet $\mathcal E$ by using transform~\eqref{eq:transform}. Each essential matrix from $\mathcal E$ can be represented in form $E_{ij} = [b_{ij}]_\times R_{ij}$, where $R_{ij} \in \mathrm{SO}(3)$ and $b_{ij}$ is a $3$-vector. We set $U_2 = U_1 R_{12}$ and $U_3 = U_2 R_{23} = U_1 R_{12}R_{23}$. Then,
\begin{equation}
\begin{split}
\widetilde E_{12} &= U_1[b_{12}]_\times R_{12} U_2^\top = [U_1b_{12}]_\times,\\
\widetilde E_{23} &= U_2[b_{23}]_\times R_{23} U_3^\top = [U_1R_{12}b_{23}]_\times,\\
\widetilde E_{31} &= U_3[b_{31}]_\times R_{31} U_1^\top = [U_1 R_{12}R_{23}b_{31}]_\times U_1 R U_1^\top,
\end{split}
\end{equation}
where $R = R_{12}R_{23}R_{31}$. Matrix $U_1$ is chosen so that\footnote{Here we use that essential matrices are real and hence so is matrix~$R$. In complex case, representation~\eqref{eq:reprR} holds if and only if the rotation axis $s$ of matrix $R$ satisfies $s^\top s \neq 0$.}
\begin{equation}
\label{eq:reprR}
U_1RU_1^\top = \begin{bmatrix}\lambda & \mu & 0 \\ -\mu & \lambda & 0 \\ 0 & 0 & 1 \end{bmatrix},
\end{equation}
where $\lambda^2 + \mu^2 = 1$ and also $(U_1R_{12}R_{23}b_{31})_1 = 0$. As a result, it suffices to prove the ``if'' part for the triplet
\begin{multline}
\label{eq:tildeE}
\widetilde E_{12} = \begin{bmatrix} 0 & -\gamma_1 & \beta_1 \\ \gamma_1 & 0 & -\alpha_1 \\ -\beta_1 & \alpha_1 & 0 \end{bmatrix},\quad
\widetilde E_{23} = \begin{bmatrix} 0 & -\gamma_2 & \beta_2 \\ \gamma_2 & 0 & -\alpha_2 \\ -\beta_2 & \alpha_2 & 0 \end{bmatrix},\\
\widetilde E_{31} = \begin{bmatrix}\gamma_3\mu & -\gamma_3\lambda & \beta_3 \\ \gamma_3\lambda & \gamma_3\mu & 0 \\ -\beta_3\lambda & -\beta_3\mu & 0 \end{bmatrix},
\end{multline}
where $\lambda, \mu, \alpha_1, \ldots, \gamma_3$ are some scalars. For the purpose of completeness, we write down the epipoles corresponding to triplet~\eqref{eq:tildeE}:
\begin{equation}
\begin{aligned}
e_{12} &= \begin{bmatrix} \alpha_1 & \beta_1 & \gamma_1 \end{bmatrix}^\top, & e_{13} &= \begin{bmatrix} -\beta_3\mu & \beta_3\lambda & \gamma_3 \end{bmatrix}^\top,\\
e_{21} &= \begin{bmatrix} \alpha_1 & \beta_1 & \gamma_1 \end{bmatrix}^\top, & e_{23} &= \begin{bmatrix} \alpha_2 & \beta_2 & \gamma_2 \end{bmatrix}^\top,\\
e_{31} &= \begin{bmatrix} 0 & \beta_3 & \gamma_3 \end{bmatrix}^\top, & e_{32} &= \begin{bmatrix} \alpha_2 & \beta_2 & \gamma_2 \end{bmatrix}^\top.
\end{aligned}
\end{equation}

Let us define an ideal $J \subset \mathbb C[\lambda, \mu, \alpha_1, \ldots, \gamma_3]$ generated by all polynomials from~\eqref{eq:necE1}~--~\eqref{eq:necE5} for triplet $\{\widetilde E_{12}, \widetilde E_{23}, \widetilde E_{31}\}$ and also by polynomial $\lambda^2 + \mu^2 - 1$. Ideal $J$ determines an affine variety $\mathcal V(J) \subset \mathbb C^{10}$. The rest of the proof consists in showing that each point of $\mathcal V(J)$ is either a compatible or degenerate triplet of essential matrices. By degeneration we mean that at least one essential matrix from the triplet is a zero matrix. For the reader's convenience, the main steps of the further proof are schematically shown in Fig.~\ref{fig:scheme}.

We consider the three main cases: (i) $\mu = 0$, $\lambda = 1$; (ii) $\mu = 0$, $\lambda = -1$; (iii) $\mu \neq 0$.

\medskip
\noindent\textbf{Case I: $\mu = 0$, $\lambda = 1$.}

First we note that a polynomial ideal and its radical define the same affine variety, whereas the structure of the radical may be much easier. For example, the Gr\"obner bases of ideal $J$ and its radical $\sqrt J$ w.r.t. the same monomial ordering consist of $217$ and $62$ polynomials respectively. Besides, there exist a simple radical membership test allowing to check whether a given polynomial belongs to the radical or not, see Lemma~\ref{lem:sqrtJ} from the Appendix. This explains why we use the radicals of ideals throughout the further proof.

Let us define the polynomials
\begin{equation}
\begin{alignedat}{2}
f_1 &= \alpha_1 + \alpha_2, &\quad g_3 &= \alpha_2\beta_3,\\
f_2 &= \beta_1 + \beta_2 + \beta_3, &\quad g_4 &= \alpha_2\gamma_3,\\
f_3 &= \gamma_1 + \gamma_2 + \gamma_3, &\quad g_5 &= \beta_1\gamma_2 - \beta_2\gamma_1,\\
g_1 &= \alpha_1\beta_3, &\quad g_6 &= \beta_2\gamma_3 - \beta_3\gamma_2,\\
g_2 &= \alpha_1\gamma_3, &\quad g_7 &= \beta_1\gamma_3 - \beta_3\gamma_1.
\end{alignedat}
\end{equation}
Then, by the radical membership test, we get
\begin{equation}
\label{eq:caseIpols}
f_ig_j \in \sqrt{J + \langle \mu, \lambda - 1 \rangle}
\end{equation}
for all indices $i$ and $j$.

First suppose that $f_i \neq 0$ for a certain $i$. Then it follows from~\eqref{eq:caseIpols} that $\alpha_1\beta_3 = \alpha_1\gamma_3 = \alpha_2\beta_3 = \alpha_2\gamma_3 = 0$. If $\alpha_1 \neq 0$ or $\alpha_2 \neq 0$, then $\beta_3 = \gamma_3 = 0$ and we get $\widetilde E_{31} = 0_{3\times 3}$ in contradiction with the rank-two condition. Therefore $\alpha_1 = \alpha_2 = 0$ and it follows from~\eqref{eq:caseIpols} that
\begin{equation}
\label{eq:caseIbetagamma}
\beta_1\gamma_2 - \beta_2\gamma_1 = \beta_2\gamma_3 - \beta_3\gamma_2 = \beta_1\gamma_3 - \beta_3\gamma_1 = 0.
\end{equation}

The variables $\beta_3$ and $\gamma_3$ cannot be zero simultaneously. Without loss of generality, suppose that $\gamma_3 \neq 0$ and introduce parameter $\delta = \beta_3/\gamma_3$. Then Eqs.~\eqref{eq:caseIbetagamma} imply $\beta_i = \delta\gamma_i$ for all~$i$. Using the radical membership test, one verifies that the variables $\gamma_i$ are constrained by
\begin{equation}
\label{eq:caseIgamma}
(\gamma_1 + \gamma_2 + \gamma_3)(-\gamma_1 + \gamma_2 + \gamma_3)(\gamma_1 - \gamma_2 + \gamma_3)(\gamma_1 + \gamma_2 - \gamma_3) = 0,
\end{equation}
that is $\gamma_3 = \epsilon_1\gamma_1 + \epsilon_2\gamma_2$ with $\epsilon_i = \pm 1$. The triplet of essential matrices has the form~\eqref{eq:triplet10} and is compatible by Lemma~\ref{lem:triplets3}.

Consider the case $f_1 = f_2 = f_3 = 0$, that is
\begin{equation}
\alpha_1 + \alpha_2 = \beta_1 + \beta_2 + \beta_3 = \gamma_1 + \gamma_2 + \gamma_3 = 0.
\end{equation}
The triplet of essential matrices has the form~\eqref{eq:triplet1} and is compatible by Lemma~\ref{lem:triplets}.

\medskip
\noindent\textbf{Case II: $\mu = 0$, $\lambda = -1$.}

Let $J' = J + \langle \mu, \lambda + 1 \rangle$. Ideal $\sqrt{J'}$ contains the following polynomials:
\begin{equation}
\begin{aligned}
&\alpha_i\beta_3\gamma_1, &\quad &\alpha_i\beta_3(\alpha_1 - \alpha_2),\\
&\alpha_i\beta_3\gamma_2, &\quad &\alpha_i\beta_3(\beta_1 - \beta_2 + \beta_3),\\
&\alpha_i\beta_3\gamma_3, &\quad &\alpha_i\beta_3\beta_1\beta_2(\beta_1\beta_2 + \alpha_1^2),
\end{aligned}
\end{equation}
where $i = 1, 2$. Supposing that $\alpha_1\beta_3 \neq 0$ or $\alpha_2\beta_3 \neq 0$, we get
\begin{multline}
\label{eq:case2eqs}
\gamma_1 = \gamma_2 = \gamma_3 = \alpha_1 - \alpha_2 = \beta_1 - \beta_2 + \beta_3 \\= \beta_1\beta_2(\beta_1\beta_2 + \alpha_1^2) = \beta_1\beta_2(\beta_1\beta_2 + \alpha_2^2) = 0.
\end{multline}
The case $\beta_1 = \beta_2 = 0$ contradicts to the rank-two condition, since leads to $\widetilde E_{31} = 0_{3\times 3}$. If $\beta_1 = 0$ and $\beta_2 \neq 0$, then triplet $\{\widetilde E_{12}, \widetilde E_{23}, \widetilde E_{31}\}$ has the form~\eqref{eq:triplet2} and is compatible by Lemma~\ref{lem:triplets}. Similarly we get compatible triplet~\eqref{eq:triplet2_1} in the case $\beta_2 = 0$ and $\beta_1 \neq 0$. Finally, if $\beta_1\beta_2 \neq 0$, then it follows from~\eqref{eq:case2eqs} that $\beta_1\beta_2 + \alpha_1^2 = 0$ and so $\beta_2 = -\alpha_1^2/\beta_1$. Triplet $\{\widetilde E_{12}, \widetilde E_{23}, \widetilde E_{31}\}$ has the form~\eqref{eq:triplet3} and is compatible by Lemma~\ref{lem:triplets}.

Now consider the case $\alpha_1\beta_3 = \alpha_2\beta_3 = 0$. There are two possibilities. If $\beta_3 = 0$, then ideal $\sqrt{J' + \langle \beta_3 \rangle}$ contains the following polynomials:
\begin{equation}
\begin{aligned}
&\gamma_3(\alpha_1 + \alpha_2), &\quad &\alpha_2\gamma_3(\gamma_1 + \gamma_2 - \gamma_3), \\
&\gamma_3(\beta_1 + \beta_2), &\quad &\beta_2\gamma_3(\gamma_1 + \gamma_2 - \gamma_3).
\end{aligned}
\end{equation}
The case $\gamma_3 = 0$ contradicts to the rank-two condition. Thus, $\alpha_1 + \alpha_2 = \beta_1 + \beta_2 = 0$. It follows that if $\alpha_2 = \beta_2 = 0$, then $\alpha_1 = \beta_1 = 0$ as well. By the radical membership test, variables $\gamma_1$, $\gamma_2$, and $\gamma_3$ are constrained by~\eqref{eq:caseIgamma} and hence we get a particular case ($\delta = 0$) of triplet~\eqref{eq:triplet10} which is compatible by Lemma~\ref{lem:triplets3}. On the other hand, if $\alpha_2 \neq 0$ or $\beta_2 \neq 0$, then we get
\begin{equation}
\alpha_1 + \alpha_2 = \beta_1 + \beta_2 = \gamma_1 + \gamma_2 - \gamma_3 = 0.
\end{equation}
Triplet $\{\widetilde E_{12}, \widetilde E_{23}, \widetilde E_{31}\}$ is a particular case ($\beta_3 = 0$) of~\eqref{eq:triplet1} which is compatible by Lemma~\ref{lem:triplets}.

The second possibility is $\beta_3 \neq 0$ and $\alpha_1 = \alpha_2 = 0$. Denote $J'' = J' + \langle \alpha_1, \alpha_2 \rangle$ and define the polynomials
\begin{equation}
\begin{split}
h_1 &= \beta_1\beta_2\beta_3\,(\beta_1(\beta_1 + \beta_2 - \beta_3) + \gamma_1(\gamma_1 + \gamma_2 + \gamma_3)),\\
h_2 &= \beta_1\beta_2\beta_3\,(\beta_2(\beta_1 + \beta_2 + \beta_3) + \gamma_2(\gamma_1 + \gamma_2 + \gamma_3)),\\
h_3 &= \beta_1\beta_2\beta_3\,(\beta_3(-\beta_1 + \beta_2 + \beta_3) + \gamma_3(\gamma_1 + \gamma_2 + \gamma_3)).
\end{split}
\end{equation}
The radical membership test yields
\begin{equation}
\label{eq:caseIIpols}
h_i \in \sqrt{J''}
\end{equation}
for all $i$.

Consider the case $\beta_1\beta_2 = 0$. If $\beta_1 = 0$, then ideal $\sqrt{J'' + \langle \beta_1 \rangle}$ contains the polynomials
\begin{equation}
\gamma_1(\beta_2 + \beta_3), \quad \gamma_1\beta_3(\gamma_1 - \gamma_2 - \gamma_3).
\end{equation}
The case $\gamma_1 = 0$ leads to $\widetilde E_{12} = 0_{3\times 3}$ and hence contradicts to the rank-two condition. Therefore we get
\begin{equation}
\alpha_1 = \alpha_2 = \beta_1 = \beta_2 + \beta_3 = \gamma_1 - \gamma_2 - \gamma_3 = 0.
\end{equation}
Triplet $\{\widetilde E_{12}, \widetilde E_{23}, \widetilde E_{31}\}$ has the form~\eqref{eq:triplet4} and is compatible by Lemma~\ref{lem:triplets}. Similarly we get compatible triplet~\eqref{eq:triplet4_1} if $\beta_2 = 0$.

Finally, let $\beta_1\beta_2 \neq 0$. The case $\gamma_1 + \gamma_2 + \gamma_3 = 0$ is impossible, since ideal $\sqrt{J'' + \langle \gamma_1 + \gamma_2 + \gamma_3 \rangle}$ contains $\beta_1\beta_2\beta_3 \neq 0$. Let us denote $\delta = -1/(\gamma_1 + \gamma_2 + \gamma_3)$. Then it follows from~\eqref{eq:caseIIpols} that
\begin{equation}
\begin{split}
\gamma_1 &= \delta\beta_1(\beta_1 + \beta_2 - \beta_3),\\
\gamma_2 &= \delta\beta_2(\beta_1 + \beta_2 + \beta_3),\\
\gamma_3 &= \delta\beta_3(-\beta_1 + \beta_2 + \beta_3).
\end{split}
\end{equation}
Triplet $\{\widetilde E_{12}, \widetilde E_{23}, \widetilde E_{31}\}$ has the form~\eqref{eq:triplet5} and is compatible by Lemma~\ref{lem:triplets}.

\medskip
\noindent\textbf{Case III: $\mu \neq 0$.}

Ideal $\sqrt J$ contains the following polynomials:
\begin{equation}
\mu\alpha_2\beta_3\gamma_1, \quad \mu\alpha_2\gamma_3, \quad \mu\alpha_2\beta_3\gamma_2.
\end{equation}
Since $\mu \neq 0$, we have $\alpha_2\gamma_3 = \beta_2\gamma_3 - \beta_3\gamma_2 = 0$.

First suppose that $\gamma_i \neq 0$ for a certain $i$. Then we get $\alpha_2\beta_3 = 0$. If $\alpha_2 \neq 0$, then $\beta_3 = \gamma_3 = 0$ and hence $\widetilde E_{31} = 0_{3\times 3}$ in contradiction with the rank-two condition. If $\alpha_2 = 0$, then ideal $\sqrt{J + \langle \alpha_2 \rangle}$ contains $\alpha_1\beta_2$ and $\alpha_1\gamma_2$. If $\alpha_1 \neq 0$, then $\widetilde E_{23} = 0_{3\times 3}$ in contradiction with the rank-two condition. Continuing this way one concludes that $\beta_1 = \beta_2 = \beta_3 = 0$. Ideal $\sqrt{J + \langle \alpha_1, \alpha_2, \beta_1, \beta_2, \beta_3 \rangle}$ contains $\mu\gamma_1\gamma_2\gamma_3$. Condition $\mu \neq 0$ implies that at least one $\gamma_i = 0$ and so at least one essential matrix $\widetilde E_{ij}$ is a zero matrix. This contradicts to the rank-two condition.

Consider the case $\gamma_i = 0$ for all $i$. Ideal $\sqrt{J + \langle \gamma_1, \gamma_2, \gamma_3 \rangle}$ contains $\beta_3(\alpha_1\lambda + \beta_1\mu + \alpha_2)$. Since $\beta_3 \neq 0$ (otherwise $\widetilde E_{31} = 0_{3\times 3}$), we conclude that $\alpha_1\lambda + \beta_1\mu + \alpha_2 = 0$. Denote $J' = J + \langle \gamma_1, \gamma_2, \gamma_3, \alpha_1\lambda + \beta_1\mu + \alpha_2\rangle$ and define the polynomials
\begin{equation}
\begin{alignedat}{2}
p_1 &= \alpha_1 - \alpha_2, &\quad q_3 &= \beta_3\lambda - \beta_1 + \beta_2,\\
p_2 &= \beta_1 - \beta_2 + \beta_3, &\quad r_1 &= \mu(\alpha_1\mu - \beta_1\lambda - \beta_2 + \beta_3),\\
p_3 &= \alpha_1\mu - \beta_1(\lambda - 1), &\quad r_2 &= \alpha_2\beta_3 + \beta_1\alpha_2 - \alpha_1\beta_2,\\
q_1 &= \alpha_2\mu - \beta_2(\lambda - 1), &\quad r_3 &= \mu(\alpha_2(\alpha_1 + \alpha_2)\\
q_2 &= \beta_3\mu + \alpha_1 - \alpha_2, & &\qquad+ \beta_2(\beta_1 + \beta_2 - \beta_3)).
\end{alignedat}
\end{equation}
Then, by the radical membership test, we get
\begin{equation}
\label{eq:caseIIIpols}
\alpha_2\beta_3p_iq_jr_k \in \sqrt{J'}
\end{equation}
for all indices $i$, $j$, and $k$.

Suppose that $\alpha_2\beta_3 = 0$. Since $\beta_3 \neq 0$ (otherwise $\widetilde E_{31} = 0_{3\times 3}$), we conclude that $\alpha_2 = 0$. Ideal $\sqrt{J' + \langle \alpha_2 \rangle}$ contains $\alpha_1\beta_2$. Since $\beta_2 \neq 0$ (otherwise $\widetilde E_{23} = 0_{3\times 3}$), we have $\alpha_1 = 0$. Ideal $\sqrt{J' + \langle \alpha_1, \alpha_2 \rangle}$ contains $\mu\beta_1$. Since $\mu \neq 0$ we have $\beta_1 = 0$ and thus $\widetilde E_{12} = 0_{3\times 3}$ in contradiction with the rank-two condition.

Let $\alpha_2\beta_3 \neq 0$. Then it follows from~\eqref{eq:caseIIIpols} that
\begin{equation}
p_1q_1r_1 = 0.
\end{equation}
Each of the three cases $p_1 = q_1 = 0$, $p_1 = r_1 = 0$, and $q_1 = r_1 = 0$ is impossible, since leads to $\mu\alpha_2\beta_3 = 0$ in contradiction with $\mu \neq 0$ and $\alpha_2\beta_3 \neq 0$.

Suppose that $p_1 = 0$, $q_1 \neq 0$, and $r_1 \neq 0$. Then it follows from~\eqref{eq:caseIIIpols} that $p_i = 0$ for all~$i$. Triplet $\{\widetilde E_{12}, \widetilde E_{23}, \widetilde E_{31}\}$ has the form~\eqref{eq:triplet7} and is compatible by Lemma~\ref{lem:triplets2}.

Let $q_1 = 0$, $p_1 \neq 0$, and $r_1 \neq 0$. Then we have $q_j = 0$ for all~$j$. Triplet $\{\widetilde E_{12}, \widetilde E_{23}, \widetilde E_{31}\}$ has the form~\eqref{eq:triplet8} and is compatible by Lemma~\ref{lem:triplets2}.

Finally, if $r_1 = 0$, $p_1 \neq 0$, and $q_1 \neq 0$, then $r_k = 0$ for all~$k$. Taking into account that $\mu \neq 0$, after some computation, we conclude that triplet $\{\widetilde E_{12}, \widetilde E_{23}, \widetilde E_{31}\}$ has the form~\eqref{eq:triplet9} and is compatible by Lemma~\ref{lem:triplets2}. We note that the denominator $p_3 = \alpha_1\mu - \beta_1(\lambda - 1)$ in~\eqref{eq:beta23} does not vanish, since otherwise $\sqrt{J' + \langle r_1, r_2, r_3, p_3 \rangle}$ contains $\mu\alpha_2\beta_3 \neq 0$.

To summarize, we have shown that in all cases the triplet of real rank-two essential matrices $\{\widetilde E_{12}, \widetilde E_{23}, \widetilde E_{31}\}$ that satisfies Eqs.~\eqref{eq:necE1}~--~\eqref{eq:necE5} is compatible. It follows from~\eqref{eq:transform} that triplet $\{E_{12}, E_{23}, E_{31}\}$ is compatible too as required. Theorem~\ref{thm:sufE} is proved.
\end{proof}

\begin{figure*}
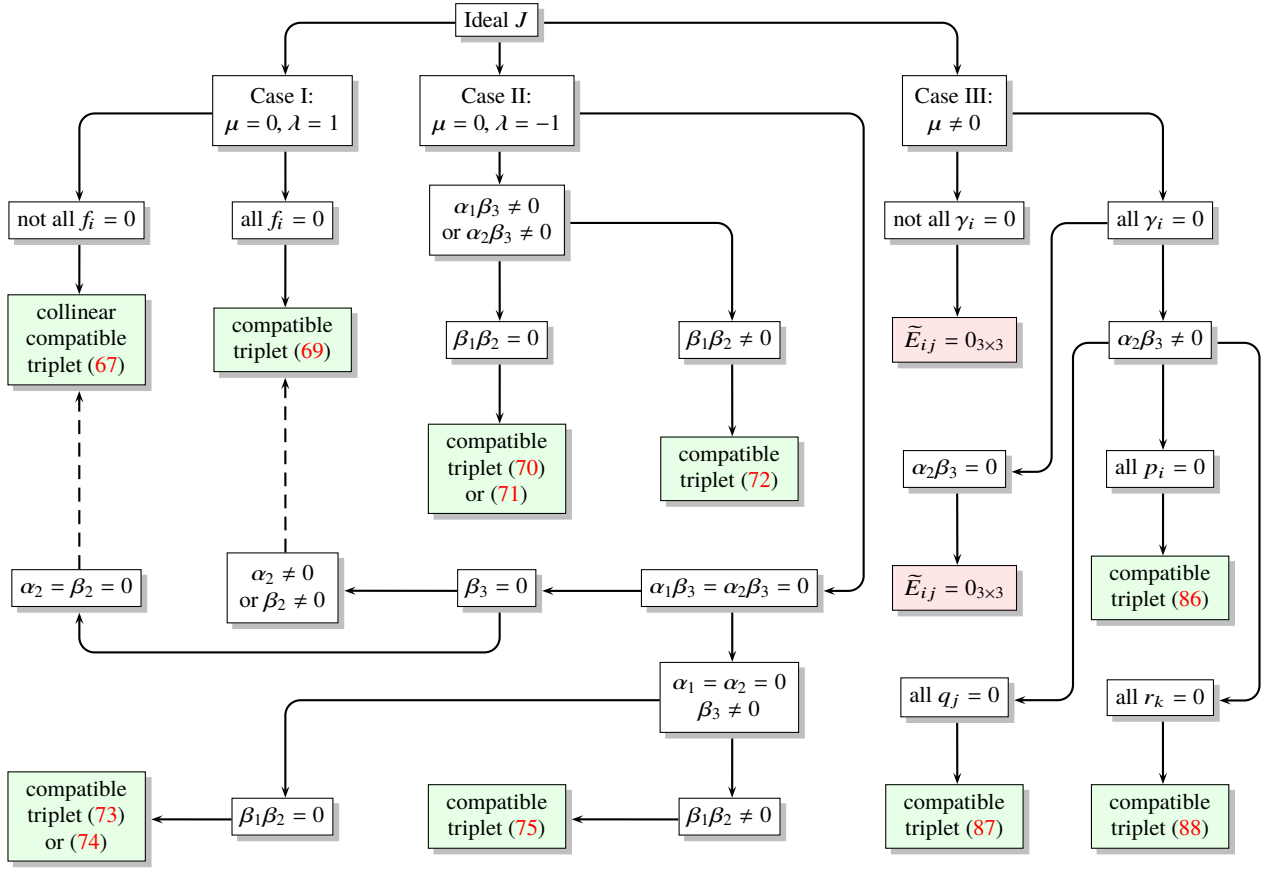

\centering
\hspace{0cm}
\begin{psmatrix}[mnode=r,colsep=0.8,rowsep=0.4]
    & & [name=J] \bx{Ideal $J$}\\
    & [name=cI] \bx{$\begin{array}{c}\text{Case I:} \\ \mu = 0, \lambda = 1\end{array}$} & [name=cII] \bx{$\begin{array}{c}\text{Case II:} \\ \mu = 0, \lambda = -1\end{array}$} & & [name=cIII] \bx{$\begin{array}{c}\text{Case III:} \\ \mu \neq 0\end{array}$}\\
    [name=cI2] \bx{not all $f_i = 0$} & [name=cI1] \bx{all $f_i = 0$} & [name=cII1] \bx{$\begin{array}{c}\alpha_1\beta_3 \neq 0\\ \text{or } \alpha_2\beta_3 \neq 0\end{array}$} &  & [name=cIII1] \bx{not all $\gamma_i = 0$} & [name=cIII2] \bx{all $\gamma_i = 0$}\\
    [name=deg1] \collin{eq:triplet10} & [name=comp1] \compat{eq:triplet1} & [name=cII12] \bx{$\beta_1\beta_2 = 0$} & [name=cII13] \bx{$\beta_1\beta_2 \neq 0$} & [name=deg3] \degen & [name=cIII22] \bx{$\alpha_2\beta_3 \neq 0$}\\
     &  & [name=comp2_1] \compator{eq:triplet2}{eq:triplet2_1} & [name=comp3] \compat{eq:triplet3} & [name=cIII21] \bx{$\alpha_2\beta_3 = 0$} & [name=cIII221] \bx{all $p_i = 0$}\\
    [name=cII211] \bx{$\alpha_2 = \beta_2 = 0$} & [name=cII212] \bx{$\begin{array}{c}\alpha_2 \neq 0\\ \text{or } \beta_2 \neq 0\end{array}$} & [name=cII21] \bx{$\beta_3 = 0$} & [name=cII2] \bx{$\alpha_1\beta_3 = \alpha_2\beta_3 = 0$} & [name=deg4] \degen & [name=comp7] \compat{eq:triplet7}\\
    &  &  & [name=cII22] \bx{$\begin{array}{c}\alpha_1 = \alpha_2 = 0\\ \beta_3 \neq 0\end{array}$} & [name=cIII222] \bx{all $q_j = 0$} & [name=cIII223] \bx{all $r_k = 0$}\\
    [name=comp5] \compator{eq:triplet4}{eq:triplet4_1} & [name=cII221] \bx{$\beta_1\beta_2 = 0$} & [name=comp6] \compat{eq:triplet5} & [name=cII222] \bx{$\beta_1\beta_2 \neq 0$} & [name=comp8] \compat{eq:triplet8} & [name=comp9] \compat{eq:triplet9}
    \ncangle[angleA=180,angleB=90,arm=10pt]{->}{J}{cI}
    \ncline{->}{J}{cII}
    \ncangle[angleA=0,angleB=90,arm=10pt]{->}{J}{cIII}
    \ncangle[angleA=180,angleB=90,arm=15pt]{->}{cI}{cI2}
    \ncline{->}{cI}{cI1}
    \ncline{->}{cII}{cII1}
    \ncangle[angleA=0,angleB=0,arm=15pt]{->}{cII}{cII2}
    \ncline{->}{cIII}{cIII1}
    \ncangle[angleA=0,angleB=90,arm=15pt]{->}{cIII}{cIII2}
    \ncline{->}{cI1}{comp1}
    \ncline{->}{cI2}{deg1}
    \ncangle[angleA=0,angleB=90,arm=15pt]{->}{cII1}{cII13}
    \ncline{->}{cII1}{cII12}
    \ncline{->}{cII12}{comp2_1}
    \ncline{->}{cII13}{comp3}
    \ncline{->}{cIII1}{deg3}
    \ncline{->}{cII2}{cII21}
    \ncline{->}{cII2}{cII22}
    \ncangle[angleA=270,angleB=270,arm=15pt]{->}{cII21}{cII211}
    \ncline{->}{cII21}{cII212}
    \ncline[linestyle=dashed]{->}{cII211}{deg1}
    \ncline[linestyle=dashed]{->}{cII212}{comp1}
    \ncline{->}{cII22}{cII222}
    \ncangle[angleA=180,angleB=90,arm=15pt]{->}{cII22}{cII221}
    \ncline{->}{cII221}{comp5}
    \ncline{->}{cII222}{comp6}
    \ncangle[angleA=180,angleB=0,arm=15pt]{->}{cIII2}{cIII21}
    \ncline{->}{cIII2}{cIII22}
    \ncline{->}{cIII22}{cIII221}
    \ncangle[angleA=180,angleB=0,arm=22pt]{->}{cIII22}{cIII222}
    \ncangle[angleA=0,angleB=0,arm=15pt]{->}{cIII22}{cIII223}
    \ncline{->}{cIII21}{deg4}
    \ncline{->}{cIII221}{comp7}
    \ncline{->}{cIII222}{comp8}
    \ncline{->}{cIII223}{comp9}
\end{psmatrix}
\caption{To the proof of Theorem~\ref{thm:sufE}. Every real point of the variety defined by the ideal $J$ is either a compatible triplet or a degenerate triplet with at least one $\widetilde E_{ij} = 0_{3\times 3}$. The dashed arrow means correspondence to a particular case}
\label{fig:scheme}
\end{figure*}

\begin{remark}
Although the $82$ cubic equations from Theorem~\ref{thm:sufE} are linearly independent, some of these equations may be algebraically dependent and therefore redundant. It is clear that if an ideal generated by a certain subset of Eqs.~\eqref{eq:necE1}~--~\eqref{eq:necE5} equals ideal $J$ defined in the proof, then Theorem~\ref{thm:sufE} remains valid for this subset. The equality of ideals is readily verified by computing their reduced Gr\"obner bases. In this way we found that Eq.~\eqref{eq:necE1} and $27$ equations of type~\eqref{eq:necE2} for indices $(i, j, k) \in \{(1, 3, 2), (3, 2, 1), (2, 1, 3)\}$ are redundant and may be omitted. Theorem~\ref{thm:sufE} holds for the remaining $56$ equations.

We also note that Eq.~\eqref{eq:necE4} only affects the compatibility of triplets with collinear epipoles, that is without Eq.~\eqref{eq:necE4} Theorem~\ref{thm:sufE} remains valid for a triplet of real rank-two essential matrices with non-collinear epipoles.
\end{remark}

\section{Applications}
\label{sec:application}

Among the possible applications of Theorem~\ref{thm:sufE} we outline the following two ones.

\subsection{Three-view Auto-Calibration}

Let $\{F_{12}, F_{23}, F_{31}\}$ be a compatible triplet of fundamental matrices and $K_i$ be the calibration matrix of the $i$th camera. Then we can write $\lambda_i \lambda_j E_{ij} = K_i^\top F_{ij} K_j$ for certain scalars $\lambda_1, \ldots, \lambda_3$, cf.~\eqref{eq:F}. Substituting this into Eqs.~\eqref{eq:necE1}~--~\eqref{eq:necE5}, we get by a straightforward computation the following equations:
\begin{align}
\label{eq:auto1}
\tr(C_1F_{12}C_2F_{23}C_3F_{31}) &= 0,\\
\label{eq:auto2}
F_{ij}^\top C_i F_{ij} C_j F_{jk} - \frac{1}{2} \tr(C_j F_{ij}^\top C_i F_{ij})\, F_{jk}\qquad \notag\\+ C_j^*\, F_{ij}^* F_{ki}^\top &= 0_{3\times 3},\\
\label{eq:auto3}
C_kF_{jk}^\top F_{ij}^* + F_{jk}^* F_{ij}^\top C_i + (F_{ij} C_j F_{jk}) \diamond F_{ki}^\top &= 0_{3\times 3},\\
\label{eq:auto5}
\tr^3((CF)^2) - 12\tr((CF)^2)\tr((CF)^4)\qquad \notag\\+ 32\tr((CF)^6) &= 0,
\end{align}
where $C_i = \lambda_i(\det K_i)^{-1}(K_iK_i^\top)$, and $C = \diag(C_1, C_2, C_3)$, and $F$ is the symmetric $9\times 9$ matrix constructed from $F_{12}$, $F_{23}$, and $F_{31}$ similarly as in formula~\eqref{eq:9x9matrixE}. It is clear that $K_iK_i^\top = C_i/(C_i)_{33}$ and therefore the calibration matrix can be estimated from $C_i$ by the Cholesky factorization.

Given $F_{12}$, $F_{23}$, and $F_{31}$, only matrices $C_i$ are constrained by Eqs.~\eqref{eq:auto1}~--~\eqref{eq:auto5} and hence these equations can be used to solve the camera auto-calibration problem in three and more views. We note that Eq.~\eqref{eq:auto5} is sextic, Eq.~\eqref{eq:auto1} is cubic, Eq.~\eqref{eq:auto2} is quadratic, and Eq.~\eqref{eq:auto3} is linear in the entries of matrix $C$. The auto-calibration constraint corresponding to Eq.~\eqref{eq:necE4} is of degree $10$ in the entries of $C$. We do not write it here.

\subsection{Incremental Structure from Motion}

In incremental structure from motion a set of essential matrices arises from independently solved two-view relative pose estimation problems. Due to the noise in input data, scale ambiguity, and other factors, the estimated essential matrices are hardly compatible. Their rectification (averaging) is one of the possible approaches to overcome the incompatibility. In~\cite{Kasten19}, the authors proposed a novel solution to the essential matrix averaging problem based on their own algebraic characterization of compatible sets of essential matrices. Although the method from~\cite{Kasten19} showed good results outperforming the existing state-of-the-art solutions both in accuracy and in speed, it is not applicable to the practically important case of cameras with collinear centers.

Polynomial constraints~\eqref{eq:necE1}~--~\eqref{eq:necE5} could also be used to solve the rectification problem. In the simplest form it can be stated as follows: given a triplet of essential matrices $\hat{\mathcal E} = \{\hat E_{12}, \hat E_{23}, \hat E_{31}\}$, find a compatible triplet $\mathcal E$ and scale factors $\Lambda = \{\lambda_{12}, \lambda_{23}, \lambda_{31}\}$ so that $\|\Lambda\|^2 = 1$ and $\mathcal E$ is the closest to $\hat{\mathcal E}_\Lambda = \{\lambda_{12}\hat E_{12}, \lambda_{23}\hat E_{23}, \lambda_{31}\hat E_{31}\}$ w.r.t. the Frobenius norm. Thus we have the following polynomial optimization problem:
\begin{equation}
\label{eq:pop}
\begin{split}
&\min\limits_{\Lambda, \mathcal E}\, \|\mathcal E - \hat{\mathcal E}_\Lambda\|^2\\
&\text{subject to}\quad \mathcal E \in \mathcal V(J), \quad \|\Lambda\|^2 = 1,
\end{split}
\end{equation}
where $J$ is the homogeneous ideal generated by all forms in~\eqref{eq:necE1}~--~\eqref{eq:necE5} and $\mathcal V(J)$ is the corresponding projective variety. Problem~\eqref{eq:pop} can be further solved by using iterative methods for constrained optimization, e.g. sequential quadratic programming method, augmented Lagrangian method, etc.

\section{Discussion}
\label{sec:disc}

In this paper we propose new necessary and sufficient conditions of the compatibility of three real rank-two essential matrices (Theorems~\ref{thm:necE} and~\ref{thm:sufE}). By compatibility we mean the correspondence of the essential matrices to a certain configuration of three calibrated cameras. The conditions have the form of $82$ cubic, one quartic, and one sextic homogeneous polynomial equations. We would like to emphasize that (i) these equations relate to the calibrated case only and in general do not hold for compatible triplets of fundamental matrices; (ii) the sufficiency of the equations covers the case of cameras with collinear centers.

The possible applications of the constraints may include multiview relative pose estimation, auto-calibration, essential matrix averaging for incremental structure from motion, etc. Regarding the auto-calibration, it is worth mentioning that some of our equations (Eq.~\eqref{eq:auto3}) turn out to be linear in the entries of matrix incorporating the calibration parameters. This unexpected result could be useful in developing novel auto-calibration solutions.

\appendix
\section*{Appendix}

We collect here several technical lemmas that we used in the proof of Theorem~\ref{thm:sufE}.

Recall that the \emph{radical} of an ideal $J$, denoted $\sqrt J$, is given by the set of polynomials which have a power belonging to~$J$:
\[
\sqrt J = \{p \mid p^k \in J \text{ for some integer $k \geq 1$}\}.
\]
It is known that $\sqrt J$ is an ideal and the affine varieties of $J$ and $\sqrt J$ coincide. The following lemma gives a convenient tool to check whether a given polynomial is in the radical or not.
\begin{lemma}[\cite{CLS}]
\label{lem:sqrtJ}
Let $J = \langle p_1, \ldots, p_s \rangle \subset \mathbb C[\xi_1, \ldots, \xi_n]$ be an ideal. Then a polynomial $p \in \sqrt J$ if and only if $1 \in \widetilde J = \langle p_1, \ldots, p_s, 1 - \tau p \rangle \subset \mathbb C[\xi_1, \ldots, \xi_n, \tau]$.
\end{lemma}

By Lemma~\ref{lem:sqrtJ}, a polynomial $p \in \sqrt J$ if and only if the (reduced) Gr\"{o}bner basis of $\widetilde J$ is~$\{1\}$. In the proof of Theorem~\ref{thm:sufE}, we used the computer algebra system \textsf{Macaulay2}~\cite{macaulay} to compute the Gr\"obner bases. The computation time did not exceed $3$ seconds per basis.

\begin{lemma}
\label{lem:sufRt}
Let essential matrices $E_{12}$, $E_{23}$, $E_{31}$ be represented in form $E_{ij} = [b_{ij}]_\times R_{ij}$. If matrices $R_{ij}$ and vectors $b_{ij}$ are constrained by
\begin{align}
\label{eq:sufRt1}
R_{12}R_{23}R_{31} &= I,\\
\label{eq:sufRt2}
R_{12}^\top b_{12} + R_{23}b_{31} + b_{23} &= 0,
\end{align}
then triplet $\{E_{12}, E_{23}, E_{31}\}$ is compatible.
\end{lemma}

\begin{proof}
Let $R_{ij}$ and $b_{ij}$ satisfy Eqs.~\eqref{eq:sufRt1}~--~\eqref{eq:sufRt2}. Then Eq.~\eqref{eq:Ecompat} has the following possible solution for $R_i$ and~$b_i$:
\begin{equation}
\begin{alignedat}{2}
R_1 &= I, &\quad b_1 &= 0,\\
R_2 &= R_{12}^\top, &\quad b_2 &= -b_{12},\\
R_3 &= R_{31}, &\quad b_3 &= R_{31}^\top b_{31}.
\end{alignedat}
\end{equation}
It follows that triplet $\{E_{12}, E_{23}, E_{31}\}$ is compatible. Lemma~\ref{lem:sufRt} is proved.
\end{proof}

\begin{lemma}
\label{lem:triplets3}
The following triplet of essential matrices with collinear epipoles is compatible:
\begin{equation}
\label{eq:triplet10}
E_{12} = \gamma_1 [s]_\times,\quad
E_{23} = \gamma_2 [s]_\times,\quad
E_{31} = (\epsilon_1\gamma_1 + \epsilon_2\gamma_2) [s]_\times,
\end{equation}
where $\epsilon_i = \pm 1$, $s = \begin{bmatrix} 0 & \delta & 1 \end{bmatrix}^\top$, and $\delta$ is an arbitrary parameter.
\end{lemma}

\begin{proof}
We denote $R_0 = \begin{bmatrix} -1 & 0 & 0 \\ 0 & -\cos\psi_0 & \sin\psi_0 \\ 0 & \sin\psi_0 & \cos\psi_0 \end{bmatrix}$, where $\cos\psi_0 = \frac{1 - \delta^2}{1 + \delta^2}$ and $\sin\psi_0 = \frac{2\delta}{1 + \delta^2}$. The essential matrices from triplet~\eqref{eq:triplet10} can be written in form $E_{ij} = [b_{ij}]_\times R_{ij}$, where matrices $R_{ij}$ and vectors $b_{ij}$ are defined as follows
\begin{equation}
\label{eq:Rt10}
\begin{alignedat}{2}
b_{12} &= -\epsilon_2\gamma_1 s, &\quad R_{12} &= \begin{cases}R_0, & \epsilon_2 = 1\\ I, & \text{otherwise}\end{cases},\\
b_{23} &= -\epsilon_1\gamma_2 s, &\quad R_{23} &= \begin{cases}R_0, & \epsilon_1 = 1\\ I, & \text{otherwise}\end{cases},\\
b_{31} &= (\epsilon_2\gamma_1 + \epsilon_1\gamma_2) s, &\quad R_{31} &= \begin{cases}R_0, & \epsilon_1\epsilon_2 = -1\\ I, & \text{otherwise}\end{cases}.
\end{alignedat}
\end{equation}
It is straightforward to verify that constraints~\eqref{eq:sufRt1}~--~\eqref{eq:sufRt2} hold. By Lemma~\ref{lem:sufRt}, triplet~\eqref{eq:triplet10} is compatible. Lemma~\ref{lem:triplets3} is proved.
\end{proof}

\begin{lemma}
\label{lem:triplets}
The following triplets of essential matrices are compatible:
\begin{enumerate}
\item
\begin{multline}
\label{eq:triplet1}
E_{12} = \begin{bmatrix} 0 & -\gamma_1 & \beta_1 \\ \gamma_1 & 0 & -\alpha_1 \\ -\beta_1 & \alpha_1 & 0 \end{bmatrix},\quad
E_{23} = \begin{bmatrix} 0 & -\gamma_2 & \beta_2 \\ \gamma_2 & 0 & \alpha_1 \\ -\beta_2 & -\alpha_1 & 0 \end{bmatrix},\\
E_{31} = \begin{bmatrix}0 & \gamma_1 + \gamma_2 & -\beta_1 - \beta_2 \\ -\gamma_1 - \gamma_2 & 0 & 0 \\ \beta_1 + \beta_2 & 0 & 0 \end{bmatrix};
\end{multline}
\item
\begin{equation}
\label{eq:triplet2}
E_{12} = \begin{bmatrix} 0 & 0 & 0 \\ 0 & 0 & -\alpha_1 \\ 0 & \alpha_1 & 0 \end{bmatrix},\,
E_{23} = \begin{bmatrix} 0 & 0 & \beta_2 \\ 0 & 0 & -\alpha_1 \\ -\beta_2 & \alpha_1 & 0 \end{bmatrix},\,
E_{31} = \begin{bmatrix} 0 & 0 & \beta_2 \\ 0 & 0 & 0 \\ \beta_2 & 0 & 0 \end{bmatrix};
\end{equation}
\item
\begin{equation}
\label{eq:triplet2_1}
E_{12} = \begin{bmatrix} 0 & 0 & \beta_1 \\ 0 & 0 & -\alpha_1 \\ -\beta_1 & \alpha_1 & 0 \end{bmatrix},\,
E_{23} = \begin{bmatrix} 0 & 0 & 0 \\ 0 & 0 & -\alpha_1 \\ 0 & \alpha_1 & 0 \end{bmatrix},\,
E_{31} = \begin{bmatrix} 0 & 0 & -\beta_1 \\ 0 & 0 & 0 \\ -\beta_1 & 0 & 0 \end{bmatrix};
\end{equation}
\item
\begin{multline}
\label{eq:triplet3}
E_{12} = \begin{bmatrix} 0 & 0 & \beta_1 \\ 0 & 0 & -\alpha_1 \\ -\beta_1 & \alpha_1 & 0 \end{bmatrix},\quad
E_{23} = \begin{bmatrix} 0 & 0 & -\frac{\alpha_1^2}{\beta_1} \\ 0 & 0 & -\alpha_1 \\ \frac{\alpha_1^2}{\beta_1} & \alpha_1 & 0 \end{bmatrix},\\
E_{31} = \begin{bmatrix} 0 & 0 & -\frac{\alpha_1^2 + \beta_1^2}{\beta_1} \\ 0 & 0 & 0 \\ -\frac{\alpha_1^2 + \beta_1^2}{\beta_1} & 0 & 0 \end{bmatrix};
\end{multline}
\item
\begin{multline}
\label{eq:triplet4}
E_{12} = \begin{bmatrix} 0 & -\gamma_2 - \gamma_3 & 0 \\ \gamma_2 + \gamma_3 & 0 & 0 \\ 0 & 0 & 0 \end{bmatrix},\quad
E_{23} = \begin{bmatrix} 0 & -\gamma_2 & \beta_2 \\ \gamma_2 & 0 & 0 \\ -\beta_2 & 0 & 0 \end{bmatrix},\\
E_{31} = \begin{bmatrix} 0 & \gamma_3 & -\beta_2 \\ -\gamma_3 & 0 & 0 \\ -\beta_2 & 0 & 0 \end{bmatrix};
\end{multline}
\item
\begin{multline}
\label{eq:triplet4_1}
E_{12} = \begin{bmatrix} 0 & -\gamma_1 & \beta_1 \\ \gamma_1 & 0 & 0 \\ -\beta_1 & 0 & 0 \end{bmatrix},\quad
E_{23} = \begin{bmatrix} 0 & -\gamma_1 - \gamma_3 & 0 \\ \gamma_1 + \gamma_3 & 0 & 0 \\ 0 & 0 & 0 \end{bmatrix},\\
E_{31} = \begin{bmatrix} 0 & \gamma_3 & \beta_1 \\ -\gamma_3 & 0 & 0 \\ \beta_1 & 0 & 0 \end{bmatrix};
\end{multline}
\item
\begin{multline}
\label{eq:triplet5}
E_{12} = \begin{bmatrix} 0 & -\gamma_1 & \beta_1 \\ \gamma_1 & 0 & 0 \\ -\beta_1 & 0 & 0 \end{bmatrix},\quad
E_{23} = \begin{bmatrix} 0 & -\gamma_2 & \beta_2 \\ \gamma_2 & 0 & 0 \\ -\beta_2 & 0 & 0 \end{bmatrix},\\
E_{31} = \begin{bmatrix} 0 & \gamma_3 & \beta_3 \\ -\gamma_3 & 0 & 0 \\ \beta_3 & 0 & 0 \end{bmatrix},
\end{multline}
where
\begin{equation}
\label{eq:gamma123}
\begin{split}
\gamma_1 &= \delta\beta_1(\beta_1 + \beta_2 - \beta_3),\\
\gamma_2 &= \delta\beta_2(\beta_1 + \beta_2 + \beta_3),\\
\gamma_3 &= \delta\beta_3(-\beta_1 + \beta_2 + \beta_3),
\end{split}
\end{equation}
and parameter $\delta$ is subject to
\begin{equation}
\label{eq:reldelta}
\delta(\gamma_1 + \gamma_2 + \gamma_3) = -1.
\end{equation}
\end{enumerate}
\end{lemma}

\begin{proof}
Triplets~\eqref{eq:triplet1}~--~\eqref{eq:triplet2_1}, \eqref{eq:triplet4}, and~\eqref{eq:triplet4_1} are compatible by definition as the corresponding essential matrices can be represented in form~\eqref{eq:Ecompat}. Namely,
\begin{enumerate}
\item for triplet~\eqref{eq:triplet1}:
\begin{equation}
E_{12} = I [b_1 - 0]_\times I,\quad E_{23} = I [0 - b_3]_\times I,\quad E_{31} = I[b_3 - b_1]_\times I,
\end{equation}
where $b_1 = \begin{bmatrix} \alpha_1 & \beta_1 & \gamma_1 \end{bmatrix}^\top$, $b_3 = \begin{bmatrix} \alpha_1 & -\beta_2 & -\gamma_2 \end{bmatrix}^\top$;

\item for triplet~\eqref{eq:triplet2}:
\begin{equation}
E_{12} = D_1 [0 - b_2]_\times I,\quad E_{23} = I [b_2 - b_3]_\times I,\quad E_{31} = I [b_3 - 0]_\times D_1,
\end{equation}
where $D_1 = \diag(1, -1, -1)$, $b_2 = \begin{bmatrix} \alpha_1 & 0 & 0 \end{bmatrix}^\top$, $b_3 = \begin{bmatrix} 0 & -\beta_2 & 0 \end{bmatrix}^\top$;

\item for triplet~\eqref{eq:triplet2_1}:
\begin{equation}
E_{12} = I [b_1 - b_2]_\times I,\quad E_{23} = I [b_2 - 0]_\times D_1,\quad E_{31} = D_1 [0 - b_1]_\times I,
\end{equation}
where $b_1 = \begin{bmatrix} 0 & \beta_1 & 0 \end{bmatrix}^\top$, $b_2 = \begin{bmatrix} -\alpha_1 & 0 & 0 \end{bmatrix}^\top$;

\item for triplet~\eqref{eq:triplet4}:
\begin{equation}
E_{12} = D_3 [b_1 - b_2]_\times I,\quad E_{23} = I [b_2 - 0]_\times I,\quad E_{31} = I [0 - b_1]_\times D_3,
\end{equation}
where $D_3 = \diag(-1, -1, 1)$, $b_1 = \begin{bmatrix} 0 & \beta_2 & -\gamma_3 \end{bmatrix}^\top$, $b_2 = \begin{bmatrix} 0 & \beta_2 & \gamma_2 \end{bmatrix}^\top$;

\item for triplet~\eqref{eq:triplet4_1}:
\begin{equation}
E_{12} = I [0 - b_2]_\times I,\quad E_{23} = I [b_2 - b_3]_\times D_3,\quad E_{31} = D_3 [b_3 - 0]_\times I,
\end{equation}
where $b_2 = \begin{bmatrix} 0 & -\beta_1 & -\gamma_1 \end{bmatrix}^\top$, $b_3 = \begin{bmatrix} 0 & -\beta_1 & \gamma_3 \end{bmatrix}^\top$.
\end{enumerate}

Further, let
\begin{equation}
\begin{alignedat}{2}
\cos\varphi_i &= \frac{\alpha_i^2 - \beta_i^2}{\alpha_i^2 + \beta_i^2}, &\quad \sin\varphi_i &= \frac{2\alpha_i\beta_i}{\alpha_i^2 + \beta_i^2},\\
\cos\psi_i &= \frac{\beta_i^2 - \gamma_i^2}{\beta_i^2 + \gamma_i^2}, &\quad \sin\psi_i &= \frac{2\beta_i\gamma_i}{\beta_i^2 + \gamma_i^2}.
\end{alignedat}
\end{equation}
The essential matrices from triplets~\eqref{eq:triplet3} and~\eqref{eq:triplet5} admit the representation $E_{ij} = [b_{ij}]_\times R_{ij}$, where matrices $R_{ij}$ and vectors $b_{ij}$ are defined below in~\eqref{eq:Rt3} and~\eqref{eq:Rt5} respectively:
\begin{enumerate}

\item
\begin{equation}
\label{eq:Rt3}
\begin{alignedat}{2}
b_{12} &= \begin{bmatrix} -\alpha_1 \\ -\beta_1 \\ 0 \end{bmatrix}, &\quad R_{12} &= \begin{bmatrix} \cos\varphi_1 & \sin\varphi_1 & 0 \\ \sin\varphi_1 & -\cos\varphi_1 & 0 \\ 0 & 0 & -1 \end{bmatrix},\\
b_{23} &= \begin{bmatrix} -\alpha_1 \\ \frac{\alpha_1^2}{\beta_1} \\ 0 \end{bmatrix}, &\quad R_{23} &= \begin{bmatrix} -\cos\varphi_1 & -\sin\varphi_1 & 0 \\ -\sin\varphi_1 & \cos\varphi_1 & 0 \\ 0 & 0 & -1 \end{bmatrix},\\
b_{31} &= \begin{bmatrix} 0 \\ -\frac{\alpha_1^2 + \beta_1^2}{\beta_1} \\ 0 \end{bmatrix}, &\quad R_{31} &= \begin{bmatrix} -1 & 0 & 0 \\ 0 & -1 & 0 \\ 0 & 0 & 1 \end{bmatrix}.
\end{alignedat}
\end{equation}

\item
\begin{equation}
\label{eq:Rt5}
\begin{alignedat}{2}
b_{12} &= -\begin{bmatrix} 0 \\ \beta_1 \\ \gamma_1 \end{bmatrix}, &\quad R_{12} &= \begin{bmatrix} -1 & 0 & 0 \\ 0 & \cos\psi_1 & \sin\psi_1 \\ 0 & \sin\psi_1 & -\cos\psi_1 \end{bmatrix},\\
b_{23} &= -\begin{bmatrix} 0 \\ \beta_2 \\ \gamma_2 \end{bmatrix}, &\quad R_{23} &= \begin{bmatrix} -1 & 0 & 0 \\ 0 & \cos\psi_2 & \sin\psi_2 \\ 0 & \sin\psi_2 & -\cos\psi_2 \end{bmatrix},\\
b_{31} &= -\begin{bmatrix} 0 \\ \beta_3 \\ \gamma_3 \end{bmatrix}, &\quad R_{31} &= \begin{bmatrix} 1 & 0 & 0 \\ 0 & -\cos\psi_3 & \sin\psi_3 \\ 0 & -\sin\psi_3 & -\cos\psi_3 \end{bmatrix}.
\end{alignedat}
\end{equation}

\end{enumerate}
It is straightforward to verify that constraints~\eqref{eq:sufRt1}~--~\eqref{eq:sufRt2} hold for both cases. By Lemma~\ref{lem:sufRt}, triplets~\eqref{eq:triplet3} and~\eqref{eq:triplet5} are compatible. Lemma~\ref{lem:triplets} is proved.
\end{proof}

\begin{lemma}
\label{lem:triplets2}
The following triplets of essential matrices are compatible provided that $\lambda^2 + \mu^2 = 1$:
\begin{enumerate}

\item
\begin{multline}
\label{eq:triplet7}
E_{12} = \beta_1\begin{bmatrix} 0 & 0 & 1 \\ 0 & 0 & -\frac{\lambda - 1}{\mu} \\ -1 & \frac{\lambda - 1}{\mu} & 0 \end{bmatrix},
E_{23} = \begin{bmatrix} 0 & 0 & \beta_1 + \beta_3 \\ 0 & 0 & -\beta_1\frac{\lambda - 1}{\mu} \\ -\beta_1 - \beta_3 & \beta_1\frac{\lambda - 1}{\mu} & 0 \end{bmatrix},\\
E_{31} = \beta_3\begin{bmatrix} 0 & 0 & 1 \\ 0 & 0 & 0 \\ -\lambda & -\mu & 0 \end{bmatrix};
\end{multline}

\item
\begin{multline}
\label{eq:triplet8}
E_{12} = \begin{bmatrix} 0 & 0 & \beta_2 + \beta_3\lambda \\ 0 & 0 & -\alpha_1 \\ -\beta_2 - \beta_3\lambda & \alpha_1 & 0 \end{bmatrix},
E_{23} = \beta_2\begin{bmatrix} 0 & 0 & 1 \\ 0 & 0 & -\frac{\lambda - 1}{\mu} \\ -1 & \frac{\lambda - 1}{\mu} & 0 \end{bmatrix},\\
E_{31} = \beta_3\begin{bmatrix} 0 & 0 & 1 \\ 0 & 0 & 0 \\ -\lambda & -\mu & 0 \end{bmatrix},
\end{multline}
where $\alpha_1 = (\beta_2 + \beta_3(\lambda + 1))\frac{\lambda - 1}{\mu}$;

\item
\begin{multline}
\label{eq:triplet9}
E_{12} = \begin{bmatrix} 0 & 0 & \beta_1 \\ 0 & 0 & -\alpha_1 \\ -\beta_1 & \alpha_1 & 0 \end{bmatrix},
E_{23} = \begin{bmatrix} 0 & 0 & \beta_2 \\ 0 & 0 & \alpha_1\lambda + \beta_1\mu \\ -\beta_2 & -\alpha_1\lambda - \beta_1\mu & 0 \end{bmatrix},\\
E_{31} = \beta_3\begin{bmatrix} 0 & 0 & 1 \\ 0 & 0 & 0 \\ -\lambda & -\mu & 0 \end{bmatrix},
\end{multline}
where
\begin{equation}
\label{eq:beta23}
\begin{split}
\beta_2 &= -\frac{(\alpha_1(\lambda - 1) + \beta_1\mu)(\alpha_1\lambda + \beta_1\mu)}{\alpha_1\mu - \beta_1(\lambda - 1)},\\
\beta_3 &= \frac{(\alpha_1^2 + \beta_1^2)(\lambda - 1)}{\alpha_1\mu - \beta_1(\lambda - 1)}.
\end{split}
\end{equation}

\end{enumerate}
\end{lemma}

\begin{proof}
Throughout the proof, $\cos\varphi_i = \frac{\alpha_i^2 - \beta_i^2}{\alpha_i^2 + \beta_i^2}$ and $\sin\varphi_i = \frac{2\alpha_i\beta_i}{\alpha_i^2 + \beta_i^2}$. The essential matrices from triplets~\eqref{eq:triplet7}, \eqref{eq:triplet8}, and~\eqref{eq:triplet9} can be written in form $E_{ij} = [b_{ij}]_\times R_{ij}$, where matrices $R_{ij}$ and vectors $b_{ij}$ are defined below in~\eqref{eq:Rt7}, \eqref{eq:Rt8}, and~\eqref{eq:Rt9} respectively:
\begin{enumerate}

\item
\begin{equation}
\label{eq:Rt7}
\begin{alignedat}{2}
b_{12} &= -\begin{bmatrix} \alpha_1 \\ \beta_1 \\ 0 \end{bmatrix}, &\quad R_{12} &= \begin{bmatrix} \cos\varphi_1 & \sin\varphi_1 & 0 \\ \sin\varphi_1 & -\cos\varphi_1 & 0 \\ 0 & 0 & -1 \end{bmatrix},\\
b_{23} &= \begin{bmatrix} \alpha_1 \\ \beta_1 + \beta_3 \\ 0 \end{bmatrix}, &\quad R_{23} &= I,\\
b_{31} &= -\begin{bmatrix} 0 \\ \beta_3 \\ 0 \end{bmatrix}, &\quad R_{31} &= R_{12},
\end{alignedat}
\end{equation}
where $\alpha_1 = \beta_1\frac{\lambda - 1}{\mu}$;

\item
\begin{equation}
\label{eq:Rt8}
\begin{alignedat}{2}
b_{12} &= \begin{bmatrix} \alpha_1 \\ \beta_2 + \beta_3\lambda \\ 0 \end{bmatrix}, &\quad R_{12} &= I,\\
b_{23} &= -\begin{bmatrix} \alpha_2 \\ \beta_2 \\ 0 \end{bmatrix}, &\quad R_{23} &= \begin{bmatrix} \cos\varphi_2 & \sin\varphi_2 & 0 \\ \sin\varphi_2 & -\cos\varphi_2 & 0 \\ 0 & 0 & -1 \end{bmatrix},\\
b_{31} &= -\begin{bmatrix} 0 \\ \beta_3 \\ 0 \end{bmatrix}, &\quad R_{31} &= R_{23},
\end{alignedat}
\end{equation}
where $\alpha_1 = (\beta_2 + \beta_3(\lambda + 1))\frac{\lambda - 1}{\mu}$, $\alpha_2 = \beta_2\frac{\lambda - 1}{\mu}$;

\item
\begin{equation}
\label{eq:Rt9}
\begin{alignedat}{2}
b_{12} &= -\begin{bmatrix} \alpha_1 \\ \beta_1 \\ 0 \end{bmatrix}, &\quad R_{12} &= \begin{bmatrix} \cos\varphi_1 & \sin\varphi_1 & 0 \\ \sin\varphi_1 & -\cos\varphi_1 & 0 \\ 0 & 0 & -1 \end{bmatrix},\\
b_{23} &= -\begin{bmatrix} \alpha_2 \\ \beta_2 \\ 0 \end{bmatrix}, &\quad R_{23} &= \begin{bmatrix} \cos\varphi_2 & \sin\varphi_2 & 0 \\ \sin\varphi_2 & -\cos\varphi_2 & 0 \\ 0 & 0 & -1 \end{bmatrix},\\
b_{31} &= \begin{bmatrix} 0 \\ \beta_3 \\ 0 \end{bmatrix}, &\quad R_{31} &= (R_{12}R_{23})^\top = \begin{bmatrix} \lambda & \mu & 0 \\ -\mu & \lambda & 0 \\ 0 & 0 & 1 \end{bmatrix},
\end{alignedat}
\end{equation}
where $\beta_2$ and $\beta_3$ are defined in~\eqref{eq:beta23}.

\end{enumerate}
By direct computation, constraints~\eqref{eq:sufRt1}~--~\eqref{eq:sufRt2} hold for all three cases. By Lemma~\ref{lem:sufRt}, triplets~\eqref{eq:triplet7}~--~\eqref{eq:triplet9} are compatible. Lemma~\ref{lem:triplets2} is proved.
\end{proof}

\bibliographystyle{spmpsci}

\bibliography{mart_triplets}

\end{document}